\documentclass[11pt]{article}

\usepackage{fullpage}

\usepackage{natbib}

\usepackage{amssymb,amsthm,amsmath}

\newcommand{\Expect}{{\mathbb E}}

\newcommand{\diag}{\mathrm{diag}}
\newcommand{\vecrm}{\mathrm{vec}}

\newcommand{\sign}{\mathrm{sign}}

\newcommand{\poly}{\mathrm{poly}}
\newcommand*{\defeq}{\stackrel{\textup{def}}{=}}
\newcommand*{\eqdef}{\stackrel{\textup{def}}{=}}

\newcommand{\cA}{{\cal A}}
\newcommand{\cB}{{\cal B}}

\newcommand{\cH}{{\cal H}}

\newcommand{\cR}{{\cal R}}

\newcommand{\tD}{\tilde{D}}

\newcommand{\Z}{\mathbb{Z}}

\newcommand{\hD}{\hat{D}}

\newcommand{\hl}{\hat{\lambda}}

\newtheorem{theorem}{Theorem}

\newtheorem{lemma}{Lemma}

\newtheorem{definition}{Definition}

\begin{document}

\title{\bf Gradient descent
           with identity initialization \\
           efficiently learns \\
           positive definite linear transformations \\
           by deep residual networks }

\author{
Peter L. Bartlett \\
UC Berkeley \\
bartlett@cs.berkeley.edu
\and
David P. Helmbold \\
UC Santa Cruz \\
dph@soe.ucsc.edu \\
\and
Philip M. Long \\
Google \\
plong@google.com
}

\maketitle

\begin{abstract}
We analyze algorithms for approximating
a function
$f(x) = \Phi x$
mapping $\Re^d$ to $\Re^d$ using deep linear
neural networks, i.e.\ that learn a function $h$ parameterized
by matrices $\Theta_1,...,\Theta_L$ and defined by
$h(x) = \Theta_L \Theta_{L-1} ... \Theta_1 x$.  We focus
on algorithms that learn through gradient descent on the population
quadratic loss in the case that the distribution over the inputs is
isotropic.  
We provide polynomial bounds on the number of
iterations for gradient descent to approximate the
least squares matrix $\Phi$, in the case where
the initial hypothesis $\Theta_1 = ... = \Theta_L = I$ 
has excess loss bounded by a small enough
constant.  On the other hand,
we show that gradient descent fails to converge for
$\Phi$ whose distance from the identity
is a larger constant, and  we show that some forms
of regularization toward the identity in each layer do
not help.  
If $\Phi$ is symmetric positive definite,
we show that an algorithm that initializes $\Theta_i = I$
learns an $\epsilon$-approximation of $f$ 
using a number of updates polynomial in $L$,
the condition number of $\Phi$, and $\log(d/\epsilon)$.  In contrast, we show
that if the least squares matrix $\Phi$ is symmetric and has a
negative eigenvalue, then all members of a class of algorithms
that perform gradient descent with identity initialization,
and optionally regularize toward the identity in each layer, fail to
converge.  
We analyze an algorithm for the case that $\Phi$ satisfies $u^{\top}
\Phi u > 0$ for all $u$, but may not be symmetric. 
This algorithm uses
two regularizers: one that maintains the invariant $u^{\top} \Theta_L
\Theta_{L-1} ... \Theta_1 u > 0$ for all $u$, and another that ``balances''
$\Theta_1, ..., \Theta_L$ 
so that they have the same singular values.
\end{abstract}

\section{Introduction}
\label{s:intro}

Residual networks \citep{he2016deep} are deep neural networks in
which, roughly, subnetworks determine how a feature transformation
should differ from the identity, rather than how it should differ from
zero.  After enabling the winning entry in the ILSVRC 2015
classification task, they have become established as a central idea in
deep networks.

\citet{HM17} provided a theoretical analysis that shed light on 
residual networks.  
They showed that 
(a) any 
linear transformation with a positive determinant and a
bounded condition number can be approximated by a
``deep linear network'' of the form 
$f(x) = \Theta_L \Theta_{L-1} ... \Theta_1 x$, where, for
large $L$, each layer $\Theta_{i}$ is close to the identity, and
(b) for networks that compose near-identity transformations this
way, if the excess loss is large, then the gradient is steep.
\citet{BEL18} extended both results to the nonlinear case, showing
that any smooth, bi-Lipschitz map can be represented as a composition
of near-identity functions, and that a suboptimal loss in a
composition of near-identity functions implies that the functional
gradient of the loss with respect to a function in the composition
cannot be small.
These results are interesting because they suggest that, 
in many cases,
this non-convex objective may be 
efficiently optimized through gradient descent
if the layers stay close to the identity, possibly with the
help of a regularizer.

This paper describes and analyzes
such
algorithms for 
linear regression with $d$ input variables
and $d$ response variables with respect to the quadratic loss, the
same setting analyzed by 
Hardt and Ma.
We abstract away
sampling issues by analyzing an algorithm that performs gradient
descent with respect to the population loss.
We focus on the case that the distribution on the input patterns is
isotropic.  (The data may be transformed through a preprocessing step to
satisfy this constraint.)

The traditional analysis of convex optimization algorithms
\citep[see][]{BV04} provides a bound in terms of the quality
of the initial solution, together with bounds on the eigenvalues
of the Hessian of the loss.  For the non-convex problem of
this paper, 
we show that if gradient descent starts at
the identity in each layer, and if the excess loss
of that initial solution is bounded by a constant, then the Hessian remains
well-conditioned enough throughout training for successful learning.
Specifically, there is a constant $c_0$ such that, if
the excess loss of the identity (over the least squares linear map)
is at most $c_0$, then back-propagation
initialized at the identity in each layer
achieves loss within at most $\epsilon$ of optimal in time polynomial 
in $\log(1/\epsilon)$, $d$, and $L$ (Section~\ref{s:identity}).
On the other hand, we show that there is a constant
$c_1$ and a least squares matrix $\Phi$ such that the identity has excess loss
$c_1$ with respect to $\Phi$, but backpropagation with
identity initialization fails to
learn $\Phi$ (Section~\ref{s:failure}).

We also show that 
%
if the least squares matrix $\Phi$ 
is symmetric positive definite then
gradient descent with identity initialization 
achieves excess loss at most $\epsilon$
in a number of steps bounded by a polynomial in 
$\log(d/\epsilon)$, $L$ and the condition number of
$\Phi$ (Section~\ref{s:psd}).

In contrast, for any least squares matrix $\Phi$ that is symmetric but
has a negative eigenvalue, we show that no such guarantee
is possible for a wide variety of algorithms of this type:
the excess loss is forever bounded below by the square of this negative
eigenvalue.  This holds for step-and-project algorithms, and also
algorithms that initialize to the identity and regularize by
early stopping or penalizing $\sum_i || \Theta_i - I ||_F^2$
(Section~\ref{s:failure}).
Both this and the previous impossibility result
can be proved using a least squares matrix $\Phi$
with a positive determinant and
a good condition number.  Recall
that such $\Phi$ were proved by Hardt and Ma to have a
good approximation as a product of near-identity matrices --
we prove that gradient descent cannot learn them, even with
the help of regularizers that reward near-identity representations.

In Section~\ref{s:powerp} we provide a convergence guarantee for a
least squares matrix $\Phi$ that may not be symmetric, but satisfies
the positivity condition $u^{\top} \Phi u > \gamma$ for some
$\gamma > 0$ that appears in the bounds.  
We call such matrices {\em $\gamma$-positive}.
Such $\Phi$ include rotations by acute angles.  In this case,
we consider an algorithm that regularizes in addition to a near-identity initialization.  
After the gradient update, 
the
algorithm performs what we call {\em power projection},
projecting its hypothesis $\Theta_L \Theta_{L-1} ... \Theta_1$ onto the set 
of $\gamma$-positive matrices.  Second, it ``balances''
$\Theta_1,...,\Theta_L$ so that, informally, they contribute
equally to $\Theta_L \Theta_{L-1} ... \Theta_1$.  (See Section~\ref{s:powerp}
for the details.)  We view this regularizer as a theoretically
tractable proxy for regularizers that promote positivity
and balance between layers
by adding penalties.  

While, in practice, deep networks are non-linear, analysis of the
linear case can provide a tractable way to gain insight through
rigorous theoretical analysis \citep{saxe2013exact,kawaguchi2016deep,HM17}.  
We might view back-propagation in the non-linear case as an approximation to a
procedure that locally modifies the function computed by each layer in a manner
that reduces the loss as fast as possible.  If a non-linear network is
obtained by composing transformations, each of which is chosen from a
Hilbert space of functions (as in \citet{DBLP:conf/nips/DanielyFS16}), 
then a step in ``function space'' corresponds to a step
in an (infinite-dimensional) linear space of functions.  

{\bf Related work.}
The motivation for this work comes from the papers of
\citet{HM17} and \citet{BEL18}.
\citet{saxe2013exact} studied the dynamics of a continuous-time
process obtained by taking the step size of backpropagation applied to
deep linear neural networks to zero.  \citet{kawaguchi2016deep} showed
that deep linear neural networks have no suboptimal local minima.  In
the case that $L=2$, the problem studied here has a similar structure
as problems arising from low-rank approximation of matrices,
especially as regards algorithms that approximate a matrix $A$ by
iteratively improving an approximation of the form $U V$.  
For an
interesting survey on the rich literature on these algorithms, please
see \citet{ge2017no};
successful algorithms have included
a regularizer that promotes balance in the sizes of $U$ and
$V$.
\citet{taghvaei2017regularization} studied the
properties of critical points on the loss when learning deep linear
neural networks in the presence of a weight decay regularizer; they
studied networks that transform the input to the output through a
process indexed by a continuous variable, instead of through discrete
layers.  \citet{LSJR16} showed that, given regularity conditions, for
a random initialization, gradient descent converges to a local
minimizer almost surely; while their paper yields useful insights,
their regularity condition does not hold for our problem.  
Many papers
have analyzed learning of neural networks with non-linearities.  The
papers most closely related to this work analyze algorithms based on
gradient descent.  
Some of these
\citep{andoni2014learning,BG17,ge2017learning,li2017convergence,ZS0BD17,ZPS18,brutzkus2018sgd,ge2018learning}
analyze constant-depth networks.  \citet{daniely2017sgd} showed that
stochastic gradient descent learns a subclass of functions computed by
log-depth networks in polynomial time; this class includes
constant-degree polynomials with polynomially bounded coefficients.
Other theoretical treatments
of neural network learning algorithms include
\citet{lee1996efficient,arora2014provable,livni2014computational,janzamin2015beating,safran2016quality,zhang2016l1,NH17,ZLWJ17a,orhan2018skip},
although these are less closely related.

Our three upper bound analyses combine a new upper bound on the operator
norm of the Hessian of a deep linear network with the result
of Hardt and Ma that gradients are lower bounded 
in terms of the loss for near-identity matrices.
They otherwise have different outlines. 
The
bound in terms of the loss of the initial solution proceeds by showing
that the distance from each layer to the identity grows slowly enough
that the loss is reduced before the layers stray far enough to harm the
conditioning of the Hessian.  The bound for symmetric positive
definite matrices proceeds by showing that, in this case, all of the
layers are the same, and each of their eigenvalues converges to the
$L$th root of a corresponding eigenvalue of $\Phi$.  As mentioned
above, the bound for $\gamma$-positive matrices $\Phi$ is for an
algorithm that achieves favorable conditioning through regularization.

We expect that the theoretical analysis reported here will inform the
design of practical algorithms for learning non-linear deep networks.
One potential avenue for this arises from the fact that the leverage
provided by regularizing toward the identity appears to already be
provided by a weaker policy of promoting the property that the
composition of layers is (potentially asymmetric) positive definite.
Also, balancing singular values of the layers of the network aided our
analysis; an analogous balancing of Jacobians associated with various
layers may improve conditioning in practice in the non-linear case.

\section{Preliminaries}
\label{s:preliminaries}

\subsection{Setting}
\label{s:setting}

For a joint distribution $P$ with support contained
in $\Re^d \times \Re^d$ and $g:\Re^{d}\to\Re^d$,
define $\ell_P(g) = \Expect_{(X,Y) \sim P}(|| g(X)-Y ||^2/2)$.
We focus on the case that,
for $(X,Y)$ drawn from $P$, the marginal on
$X$ is isotropic, with $\Expect XX^\top=I_d$.
For convenience, we assume that $Y = \Phi X$ for 
$\Phi \in \Re^{d \times d}$. This assumption is without loss of generality: 
if $\Phi$ is the least squares matrix (so that $f$ 
defined by $f(X) = \Phi X$ minimizes
$\ell_P(f)$ among linear functions), for any linear $g$ we have
  \begin{align*}
    \ell_P(g) & = \Expect\|g(X)-f(X) \|^2/2
        + \Expect\|f(X)-Y\|^2/2 \\*
    & \qquad {}
    +\Expect\left((g(X)-f(X))(f(X)-Y)\right) \\
    & = \Expect\|g(X)-f(X) \|^2/2 + \Expect\|f(X)-Y\|^2/2 \\
    & = \Expect\|g(X)- \Phi X) \|^2/2 + \Expect\|\Phi X-Y\|^2/2,  
  \end{align*}
since $f$ is the projection of $Y$ onto
the set of linear functions of $X$.
So assuming $Y=\Phi X$ corresponds
to setting $\Phi$ as the least squares matrix and replacing the
loss $\ell_P(g)$ by the excess loss
  \[
    \Expect\|g(X)-\Phi X\|^2/2 = 
    \Expect\|g(X)-Y\|^2/2 - \Expect\|\Phi X-Y\|^2/2.
  \]

We study algorithms that learn linear mappings parameterized by deep
networks.  The network with $L$ layers and
parameters $\Theta=(\Theta_1,\ldots,\Theta_L)$ computes the
parameterized function
$
    f_{\Theta}(x) = \Theta_L\Theta_{L-1}\cdots\Theta_1 x,
$
where $x\in\Re^{d}$ and $\Theta_i\in\Re^{d\times d}$.

We use the notation
$
    \Theta_{i:j} = \Theta_j\Theta_{j-1}\cdots\Theta_i
$
for $i\le j$, so that we can write
$
    f_{\Theta} (x) = \Theta_{1:L}x = \Theta_{i+1:L}\Theta_i\Theta_{1:i-1}x.$

When there is no possibility of confusion, we will sometimes
refer to loss $\ell(f_{\Theta})$ simply as $\ell(\Theta)$.  
Because the distribution of
$X$ is isotropic,
$
\ell(\Theta) = \frac{1}{2} || \Theta_{1:L} - \Phi ||_F^2
$
with respect to least squares matrix $\Phi$.
When $\Theta$
is produced by an iterative algorithm, will we also refer to
loss of the $t$th iterate by $\ell(t)$.

\begin{definition}
\label{d:margin}
For $\gamma > 0$, a matrix $A \in \Re^{d \times d}$ 
is {\em $\gamma$-positive} 
if, for all unit length $u$, we have $u^{\top} A
u > \gamma$.
\end{definition}

\subsection{Tools and background}


We use $|| A ||_F$ for the Frobenius norm of matrix $A$,
$|| A ||_2$ for its operator norm, and $\sigma_{\min}(A)$ for its least singular value.
For vector $v$, we use $|| v ||$ for its Euclidian norm.

For a matrix $A$ and a matrix-valued function $B$, define
$D_A B(A)$ to be the matrix with
  \[
    \left(D_A B(A)\right)_{i,j} = \frac{\partial \vecrm(B(A))_i}
    {\partial \vecrm(A)_j},
  \]
where $ \vecrm(A)$ is the column vector constructed by stacking the
columns of $A$.
We use $T_{d,d}$ to denote the $d^2 \times d^2$ permutation matrix mapping 
$\vecrm(A)$ to $\vecrm(A^{\top})$ for $A \in \Re^{d \times d}$.
For $A\in\Re^{n\times m}$ and $B\in\Re^{p\times q}$,
$A\otimes B$ denotes the Kronecker product, that is, the $np\times mq$
matrix of $n\times m$ blocks, with the $i,j$th block given by $A_{ij} B$.

We will need the gradient and Hessian of $\ell$.  
(The gradient, which can be computed using backprop, is of course
well known.)  The proof is in Appendix~\ref{a:gradient.hessian}.
\begin{lemma}
\label{l:gradient.hessian}
  \begin{align*}
%
      D_{\Theta_i} \ell\left(f_{\Theta} \right) 
        & \!=\! (\vecrm( I_d ))^{\top}
         \left( \left(\Theta_{1:i-1}^{\top} \otimes (\Theta_{1:L} \!-\! \Phi)^{\top} \Theta_{i+1:L}\right) \right) \\
        & = \vecrm(G)^\top, 
  \end{align*}
where $G$ is the $d \times d$ matrix given by
\begin{equation}
\label{e:gradient.matrix}
G \eqdef \Theta_{i+1:L}^{\top} \left(\Theta_{1:L} - \Phi \right) 
         \Theta_{1:i-1}^{\top}.
\end{equation}
  For $i<j$,
\begin{align*}
D_{\Theta_j} D_{\Theta_i} \ell\left(f_{\Theta}\right) 
    & = (I_{d^2} \otimes (\vecrm( I_d ))^{\top})
      \left( I_d \otimes T_{d,d} \otimes I_d \right)
      \left( \vecrm(\Theta_{1:i-1}^{\top}) \otimes I_{d^2} \right) \\
 & \hspace{0.3in}
       (
     ( \Theta_{i+1:L}^{\top} \Theta_{j+1:L} \otimes \Theta_{1:j-1}^{\top} ) T_{d,d} 
     +
     ( \Theta_{i+1:j-1}^{\top} \otimes (\Theta_{1:L} - \Phi)^{\top} \Theta_{j+1:L} )
       ).
       \\
D_{\Theta_i} D_{\Theta_i} \ell\left(f_{\Theta}\right) 
       & = (I_{d^2} \otimes (\vecrm( I_d ))^{\top})
      \left( I_d \otimes T_{d,d} \otimes I_d \right)
      \left( \vecrm(\Theta_{1:i-1}^{\top}) \otimes I_{d^2} \right) \\
    & \hspace{0.3in}
     \left( \Theta_{i+1:L}^{\top} \Theta_{i+1:L} \otimes \Theta_{1:i-1}^{\top} \right) T_{d,d}.
  \end{align*}
\end{lemma}

\section{Targets near the identity}
\label{s:identity}

In this section, we prove an upper bound for gradient
descent in terms of the loss of the initial solution.

\subsection{Procedure and upper bound}
\label{s:identity.procedure}

First, set
$\Theta^{(0)} = (I, I, ..., I)$,
and then iteratively update
\[
\Theta_i^{(t+1)} = 
\Theta_i^{(t)} 
 - \eta (\Theta_{i+1:L}^{(t)})^{\top} \left(\Theta_{1:L}^{(t)} - \Phi \right) 
         (\Theta_{1:i-1}^{(t)})^{\top}.
\]

\begin{theorem}
\label{t:succeeds.identity}
There are positive constants $c_1$ and $c_2$ and polynomials $p_1$ and
$p_2$ such that, if $\ell(\Theta_{1:L}^{(0)}) \leq c_1$, $L \geq c_2$,
and $\eta \leq \frac{1}{p_1(L, d, || \Phi ||_2)}$, then the above
gradient descent procedure achieves $\ell(f_{\Theta^{(t)}}) \leq
\epsilon$ within $t = p_2\left(\frac{1}{\eta}\right) \ln \left(
\frac{\ell(0)}{\epsilon} \right)$ iterations.
\end{theorem}

\subsection{Proof of Theorem~\ref{t:succeeds.identity}}
\label{s:identity.analysis}


The following lemma, which is implicit in the proof of Theorem~2.2 in
\cite{HM17}, shows that the gradient is steep if the loss is large
and the singular values of the layers are not too small.
\begin{lemma}[\citealt{HM17}]
\label{l:steep}
Let $\nabla_{\Theta} \ell(\Theta)$ be
the gradient of $\ell(\Theta)$ with respect to
any flattening of $\Theta$.
If, for all layers $i$, $\sigma_{\min}(\Theta_i) \geq 1-a$,  then
$
|| \nabla_{\Theta} \ell(\Theta) ||^2 \geq 
  4 \ell(\Theta) L (1 - a)^{2 L}.$
\end{lemma}

Next, we show that, if $\Theta^{(t)}$ and
$\Theta^{(t+1)}$ are both close to the identity, then the gradient is not
changing very fast between them, so that rapid progress continues to be made.
We prove this through an upper bound
on the operator norm of the Hessian that holds uniformly over members
of a ball around the identity, which in turn can be obtained through a bound on
the Frobenius norm.  The proof 
is in Appendix~\ref{a:smooth}.
\begin{lemma}
\label{l:smooth}
Choose an arbitrary $\Theta$ with $|| \Theta_i  ||_2 \leq 1+z$ for all $i$, and 
least squares matrix $\Phi$ with $|| \Phi ||_2 \leq (1 + z)^L$.
Let
$\nabla^2$ be the Hessian of $\ell(f_{\Theta})$ with
respect to an arbitrary flattening of the parameters of $\Theta$.  We have
\[
|| \nabla^2 ||_F \leq 3 L d^{5} (1 + z)^{2 L}.
\]
\end{lemma}

Armed with Lemmas~\ref{l:steep} and \ref{l:smooth}, let us now analyze
gradient descent.  Very roughly, our strategy will be to show that the
distance from the identity to the various layers grows slowly enough
for the leverage from Lemmas~\ref{l:steep} and \ref{l:smooth} to
enable successful learning.  Let $\cR(\Theta) = \max_i || \Theta_i - I ||_2$.  
From the update, we have
\begin{align*}
 ||\Theta_i^{(t+1)} - I ||_2  
& \leq
||\Theta_i^{(t)} - I ||_2
 +
 \eta || (\Theta_{i+1:L}^{(t)})^{\top} \left(\Theta_{1:L}^{(t)} - \Phi  \right) 
         (\Theta_{1:i-1}^{(t)})^{\top} ||_2 \\
 & \leq 
   ||\Theta_i^{(t)} - I ||_2 + \eta (1 + \cR(\Theta^{(t)}))^L
                              || \Theta_{1:L}^{(t)} - \Phi ||_2 \\
 & \leq 
   ||\Theta_i^{(t)} - I ||_2 + \eta (1 + \cR(\Theta^{(t)}))^L
                              || \Theta_{1:L}^{(t)} - \Phi ||_F.
\end{align*}
If $\cR(t) = \max_{s \leq t} \cR(\Theta^{(s)}) $ (so $\cR(0)=0$) and
$\ell(t) = \frac{1}{2} || \Theta_{1:L}^{(t)} - \Phi ||_F^2$,
this implies
\begin{align} \label{e:Rdef}
\cR(t+1)
 \leq \cR(t) + \eta (1 + \cR(t))^L \sqrt{ 2 \ell(t) }.
\end{align}

By Lemma~\ref{l:smooth},
for all $\Theta$ on the line segment from $\Theta^{(t)}$ to
$\Theta^{(t+1)}$, we have
\[
|| \nabla_{\Theta}^2 ||_2 
\leq || \nabla_{\Theta}^2 ||_F
\leq 3 L d^5 \max\{(1+\cR(t+1))^{2L},  | |\Phi ||_2^{2} \}  ,
\]
so that
\begin{align*}
\ell(t+1) &  \leq \ell(t) - \eta || \nabla_{\Theta^{(t)}} ||^2 
   + \frac{3}{2} \eta^2 L d^5 \max\{(1 + \cR(t+1))^{2 L}, || \Phi ||_2^2 \} || \nabla_{\Theta^{(t)}} ||^2.
\end{align*}

Thus, if we ensure
\begin{equation}
\label{e:eta_t.necc.2}
\eta \leq \frac{ 1 }{3 L d^5 \max\{(1 + \cR(t+1) )^{2 L}, || \Phi ||_2^2 \} } ,
\end{equation}
we have
$
\ell(t+1)
   \leq \ell(t) - (\eta/2) || \nabla_{\Theta^{(t)}} ||^2,$
which, using 
Lemma~\ref{l:steep},
gives
\begin{align} \label{e:loss.recurrence}
\ell(t+1)
& \leq \left(1 - 2 \eta L (1 - \cR(t))^{2 L} \right) \ell(t).
 \end{align}

Pick any $c \geq 1$.
Assume that
 $L \geq (4/3) \ln c = c_2$,
$\ell(\Theta_{1:L}^{(0)}) \leq \frac{\ln(c)^2}{8 c^{10}} = c_1$ 
and $\eta \leq \frac{1}{3 L d^5 \max\{c^4, || \Phi ||_2^2 \} }$.
We claim that, for all $t \geq 0$,
\begin{enumerate}
\item $\cR(t) \leq \eta c \sqrt{2 \ell(0)}
     \sum_{0 \leq s < t} \exp \left(- \frac{ s \eta L }{ c^4  } \right)$

\item $\ell(t) \leq \left( \exp \left(- \frac{ 2 t \eta L }{c^4 } \right)
             \right) \ell(0)$.
\end{enumerate}
The base case holds as $\cR(0) = 0$ and $\ell(0) = \ell(0)$.

Before starting the inductive step, notice that for any $t \geq 0$, 
\begin{align*}
 \eta c \sqrt{2 \ell(0)} 
     \sum_{0 \leq s < t} \exp \left(- \frac{ s \eta L }{ c^4  } \right) 
& \leq  \eta c \sqrt{2 \ell(0)}
      \times \frac{1}{1 - \exp\left(-\frac{ \eta L}{ c^4  }\right)} \\
& \leq \eta c \sqrt{2 \ell(0)} 
      \times \frac{2 c^4}{\eta L} \;\;\;
                            & \mbox{(since $\frac{ \eta L}{ c^4  } \leq 1$)} \\
& = \frac{2 c^5 \sqrt{2 \ell(0)}}{L}  
 \leq \frac{\ln c }{L} \leq 3/4
\end{align*}
where the last two inequalities follow from the constraints on $\ell(0)$
and $L$.


Using~(\ref{e:Rdef}),
\begin{align*}
\cR(t+1)
	& \leq \cR(t) + \eta  (1 + \cR(t))^L  \sqrt{ 2 \ell(t) } \\
	&\leq  \cR(t) + \eta  \left(1 + \frac{\ln c}{L} \right)^L  \sqrt{ 2 \ell(t) } \\
	&\leq \cR(t) + \eta c \sqrt{ 2 \ell(t) } \\
	& \leq \cR(t) + \eta c  \sqrt{2 \ell(0)}   \exp \left(- \frac{  t \eta L }{ c^4 } \right) \\
	& \leq \eta c \sqrt{2 \ell(0)} \sum_{0 \leq s < t+1} \exp \left(- \frac{ s \eta L }{ c^4  } \right).
\end{align*} 


Since $\cR(t+1) \leq \frac{ \ln c }{ L}$, the choice of $\eta$ satisfies~(\ref{e:eta_t.necc.2}), so 
\[
\ell(t+1)
 \leq \left(1 - 2 \eta L (1 - \cR(t))^{2 L} \right) \ell(t) .
 \]

Now consider $(1-\cR(t))^{2L}$:
\begin{align*}
\ln \left( (1-\cR(t))^{2L} \right) &= 2L \ln(1-\cR(t)) \\
			& \geq 2L (-2\cR(t))   & \text{since } \cR(t) \in [0,3/4] \\
			 &\geq 2L \left(-2 \frac{\ln c}{L} \right)  & \text{since }\cR(t) \leq \frac{\ln c}{L} \\
(1-\cR(t))^{2L} &\geq 1/c^4.
\end{align*}

Using this in the bound on $\ell(t+1)$:
\begin{align*}
\ell(t+1)
	& \leq \left(1 - 2 \eta L (1 - \cR(t))^{2 L} \right) \ell(t) \\
	& \leq \left(1 - \frac{2 \eta L}{ c^4}\right) \ell(t) \\
	& \leq  \left(  \exp \left(- \frac{ 2 \eta L }{c^4 }  \right) \right) \left( \exp \left(- \frac{ 2 t \eta L }{c^4 } \right)
             \right) \ell(0) 			\\
         & = \left( \exp \left(- \frac{ 2 (t+1) \eta L }{c^4 } \right)
             \right) \ell(0).
\end{align*}
		  
Solving  
$\ell(0)  \exp \left(- \frac{ 2 t \eta L }{c^4 } \right)     \leq \epsilon$ for $t$ 
and recalling that 
$\eta < 1/c^4$ 
completes the proof of the theorem.


\section{Symmetric positive definite targets}
\label{s:psd}

In this section, we analyze the procedure of Section~\ref{s:identity.procedure} when the
least squares matrix $\Phi$ is symmetric and positive definite.

\begin{theorem}
\label{t:succeeds.psd}
There is an absolute positive constant $c_3$ such
that, if $\Phi$ is symmetric and $\gamma$-positive with $0<\gamma<1$, 
and $L \geq c_3 \ln\left(|| \Phi ||_2/\gamma\right)$,
then
for all $\eta \leq \frac{1}{L (1 + || \Phi ||_2^2)}$,
gradient descent achieves $\ell(f_{\Theta^{(t)}}) \leq \epsilon$
in $\mathrm{\poly} (L,|| \Phi ||_2/\gamma,1/\eta) \log (d/\epsilon)$ iterations.
\end{theorem}

Note that a symmetric matrix is $\gamma$-positive when its minimum
eigenvalue is at least $\gamma$.

\subsection{Proof of Theorem~\ref{t:succeeds.psd}}
\label{s:normal.analysis}

Let $\Phi$ be a symmetric, real, $\gamma$-positive matrix with $\gamma > 0$,
and let $\Theta^{(0)}, \Theta^{(1)}, ...$ be the iterates
of gradient descent with a step size 
$0 < \eta \leq \frac{1}{L (1 + || \Phi ||_2^2)}$.

\begin{definition}
Symmetric 
matrices $\cA \subseteq \Re^{d \times d}$ are {\em commuting normal matrices}
if there is a single unitary matrix $U$ such that
for all $A \in \cA$, $U^{\top} A U$ is diagonal.
\end{definition}

We will use the following well-known facts about 
commuting normal matrices.
\begin{lemma}[\citealt{horn2013matrix}]
\label{l:normal}
If $\cA \subseteq \Re^{d \times d}$ is a set of
symmetric commuting normal matrices and $A, B \in \cA$, the following hold:
\begin{itemize}
\item $A B = B A$;
\item for all scalars $\alpha$ and $\beta$,
$\cA \cup \{ \alpha A + \beta B, A B \}$ are commuting normal;
\item there is a unitary matrix $U$ such that
$U^{\top} A U$ and $U^{\top} B U$ are real and diagonal;
\item the multiset of singular values of $A$ is the same as the multiset of
magnitudes of its eigenvalues;
\item $|| A - I ||_2$ is the largest value of
$|z - 1|$ for an eigenvalue $z$ of $A$.
\end{itemize}
\end{lemma}

\begin{lemma}
\label{l:commute}
The matrices
$\{ \Phi \}
\cup
\{ \Theta_i^{(t)} : i \in \{ 1,...,L \}, t \in \Z^+ \}$
are commuting normal.
For all $t$,
$
\Theta_1^{(t)} = ... = \Theta_L^{(t)}.$
\end{lemma}
\begin{proof}
The proof is by induction.  The base case follows from
the fact that $\Phi$ and $I$ are commuting normal.

For the induction step, the fact that
\begin{align*}
& \{ \Phi \}
\cup
\left\{ \Theta_i^{(s)} : i \in \{ 1,...,L \}, s \leq t \right\} 
 \cup
\left\{ \Theta_i^{(s + 1)} : i \in \{ 1,...,L \}, s \leq t \right\}
\end{align*}
are commuting normal follows from Lemma~\ref{l:normal}.
The update formula now reveals that
$
\Theta_1^{(t + 1)} = ... = \Theta_L^{(t + 1)}.$
\end{proof}

Now we are ready to analyze the dynamics of the learning process.
Let $\Phi = U^{\top} D^{L} U$ be
a diagonalization of $\Phi$.  Let $\Gamma = \max\{ 1, || \Phi ||_2 \}$.
We next describe a sense in
which gradient descent learns each eigenvalue independently.
\begin{lemma}
\label{l:independent.eigenvalues}
For each $t$, there is a real diagonal matrix
$\hD^{(t)}$ such that, for all $i$,
$\Theta_i^{(t)} = U^{\top} \hD^{(t)} U$ and
\begin{equation}
\label{e:diagonal.update}
\hD^{(t+1)} = \hD^{(t)} - \eta (\hD^{(t)})^{L-1} ((\hD^{(t)})^{L} - D^{L}).
\end{equation}
\end{lemma}
\begin{proof}
Lemma~\ref{l:commute} implies that there
is a single real $U$ such that 
$\Theta_i^{(t)} = U^{\top} \hD^{(t)} U$ for all $i$.  
Applying Lemma~\ref{l:gradient.hessian}, 
recalling that
$\Theta_1^{(t)} = ... = \Theta_L^{(t)}$, and
applying the fact that $\Theta_i^{(t)}$ and $\Phi$ commute, we get
\begin{align*}
\Theta_i^{(t+1)} 
 & = \Theta_i^{(t)} 
   - \eta (\Theta_i^{(t)})^{L-1} \left( (\Theta_i^{(t)})^L - \Phi \right).
\end{align*}
Replacing each matrix by its diagonalization, we get
\begin{align*}
 U^{\top} \hD^{(t+1)} U 
& = U^{\top} \hD^{(t)} U 
- \eta (U^{\top} (\hD^{(t)})^{L-1} U) 
          \left(U^{\top} (\hD^{(t)})^{L} U - U^{\top} D^{L} U \right)  \\
& = U^{\top} \hD^{(t)} U 
      - \eta U^{\top} (\hD^{(t)})^{L-1} \left((\hD^{(t)})^{L} - D^{L} \right)  U,
\end{align*}
and left-multiplying by $U$ and right-multiplying by $U^{\top}$
gives (\ref{e:diagonal.update}).
\end{proof}

We will now analyze the 
convergence of each $\hD_{kk}^{(t)}$
to $D_{kk}$ separately.  
Let us focus for now on an
arbitrary single index $k$,
let $\lambda = D_{kk}$
and $\hl^{(t)} = \hD_{kk}^{(t)}$.  

Recalling that $|| \Phi ||_2 \leq \Gamma$,
we have
$
\gamma^{1/L} \leq \lambda \leq \Gamma^{1/L}.$
Also, $\Gamma^{1/L} = e^{\frac{1}{L} \ln \Gamma } \leq e^{1/a}\leq 1 + 2/a$ 
whenever $a \geq 1$ and $L\geq a \ln \Gamma$.
Similarly, $\gamma^{1/L} \geq 1-a$ whenever $L \geq a \ln (1/\gamma)$.
Thus, there are absolute constants
$c_3$ and $c_4$ such that
$|1 - \lambda| \leq \frac{c_4 \ln(\Gamma/\gamma)}{L} < 1$
for all $L \geq c_3 \ln(\Gamma/\gamma)$.


We claim that, for all $t$, 
$\hl^{(t)}$ lies between $1$ and $\lambda$ inclusive, so
that $|\hl^{(t)} - \lambda| \leq \frac{c_4 \ln(\Gamma/\gamma)}{L}$.
The base case holds because $\hl^{(t)} = 1$
and $|1 - \lambda| \leq \frac{c_4 \ln(\Gamma/\gamma)}{L}$.
Now let us work on the induction step.
Applying (\ref{e:diagonal.update}) together
with Lemma~\ref{l:gradient.hessian}, we get
\begin{equation} \label{e:single_eigen}
\hl^{(t+1)} = \hl^{(t)} + \eta (\hl^{(t)})^{L-1}  (\lambda^L - (\hl^{(t)})^L).
\end{equation}
By the induction hypothesis, we just need to show
that 
$\sign(\hl^{(t+1)} - \hl^{(t)}) = \sign(\lambda - \hl^{(t)})$ and
$|\hl^{(t+1)} - \hl^{(t)}| \leq |\lambda - \hl^{(t)}|$ (i.e., the step is in
the correct direction, and does not ``overshoot'').  First, to see
that the step is in the right direction, note that $\lambda^L \geq
(\hl^{(t)})^L$ if and only if $\lambda \geq (\hl^{(t)})$, and the
inductive hypothesis implies that $\hl^{(t)}$, and therefore
$(\hl^{(t)})^{L-1}$, is non-negative.  To show that $|\hl^{(t+1)} -
\hl^{(t)}| \leq |\lambda - \hl^{(t)}|$, it suffices to show that $\eta
(\hl^{(t)})^{L-1} \left| \lambda^L - (\hl^{(t)})^L) \right| \leq
|\lambda - \hl^{(t)}|$, which, in turn would be implied by $\eta \leq
\left| \frac{1}{(\hl^{(t)})^{L-1} \left(\sum_{i=0}^{L-1} (\hl^{(t)})^i
  \lambda^{L-1-i}\right)} \right|$ (since $\lambda^L - (\hl^{(t)})^{L}
= (\lambda - \hl^{(t)}) \sum_{i=0}^{L-1} (\hl^{(t)})^i
\lambda^{L-1-i}$), which follows from the inductive hypothesis and
$\eta \leq \frac{1}{L \Gamma^2}$.

We have proved that each $\hl^{(t)}$ lies between
$\lambda$ and $1$, so that 
$| 1 - \hl^{(t)} | \leq | 1 - \lambda| \leq c_4 \ln(\Gamma/\gamma)$.

Now, since the step is in the right direction, and does not overshoot,
\begin{align*}
 | \hl^{(t+1)} - \lambda | 
 & \leq | \hl^{(t)} - \lambda |
   - \eta (\hl^{(t)})^{L-1}  | \lambda^L - (\hl^{(t)})^L | \\
 & \leq | \hl^{(t)} - \lambda |
    \left(1 - \eta (\hl^{(t)})^{L-1} 
              \left(\sum_{i=0}^{L-1} (\hl^{(t)})^i \lambda^{L-1-i}\right)
       \right)  \\
 & \leq | \hl^{(t)} - \lambda |
    \left(1 - \eta L \gamma^2 \right),
\end{align*}
since the fact that $\hl^{(t)}$ lies between
$1$ and $\lambda$ implies that $\hl^{(t)} \geq \gamma^{1/L}$.
Thus, $| \hl^{(t)} - \lambda | \leq 
    \left(1 - \eta L \gamma^2 \right)^t 
    c_4 \ln(\Gamma/\gamma)$.
This implies that, for any 
$\epsilon \in (0,1)$, for any
absolute
constant $c_5$, there is a constant
$c_6$ such that, after
$c_6 \frac{1}{\eta L \gamma^2} \ln\left(\frac{d L \ln \Gamma}{\gamma \epsilon}\right)$ steps,
we have
$| \hl^{(t)} - \lambda | 
   \leq \frac{c_5 \gamma \sqrt{\epsilon}}{L \Gamma \sqrt{d}}.$
Writing $r = \hl^{(t)} - \lambda$, this implies, if
$c_5$ is small enough, that
\begin{align*}
 ((\hl^{(t)})^L - \lambda^L)^2 
&= ((\lambda \!+\! r)^L \!-\! \lambda^L)^2 
 \\
& 
\leq \Gamma^2 \left(\left(1 \!+\! \frac{r}{\lambda}\right)^L \!-\! 1\right)^2  \\
& \leq \Gamma^2 \left(\frac{2 c_5 r L}{\lambda}\right)^2 
 \\
 & 
\leq \Gamma^2 \left(\frac{2 c_5 r L}{\gamma}\right)^2 
 \\
 & 
\leq \frac{\epsilon}{d}.
\end{align*}

Thus, after
$O\left(\frac{1}{\eta L \gamma^2} \ln\left(\frac{d L \ln \Gamma}{\gamma \epsilon}\right)\right)$
steps, $(D_{kk} - \hD_{kk}^{(t)})^2 \leq \epsilon/d$
for all $k$, and therefore $\ell(\Theta^{(t)}) \leq \epsilon$,
completing the proof.

\section{Asymmetric positive definite matrices}
\label{s:powerp}

We have seen that if the least squares matrix is symmetric,
$\gamma$-positivity
is sufficient for convergence of gradient descent. We shall see in
Section 6 that positivity is also necessary for a broad family of
gradient-based algorithms to converge to the optimal solution when the
least squares matrix is symmetric. Thus, in the symmetric case,
positivity characterizes the success of gradient methods.  In this
section, we show that positivity suffices for the convergence of a
gradient method even without the assumption that the least squares
matrix is symmetric.

Note that the set of $\gamma$-positive (but not necessarily symmetric)
matrices
includes both rotations by an acute angle 
and  ``partial reflections'' of the form $a x + b \text{ refl}(x)$ where $\text{refl}(\cdot)$ is a 
length-preserving reflection and $0 \leq |b| < a$.
Since $\left(u^{\top} A u \right)^{\top} = u^{\top} A^{\top} u$, 
a matrix $A$ is $\gamma$-positive
if and only if $u^{\top} (A + A^{\top}) u \geq 2\gamma$ for
all unit length $u$, i.e.\ $A+A^{\top}$ 
is positive definite with eigenvalues at least $2 \gamma$.

\subsection{Balanced factorizations}
\label{s:balanced}

The algorithm analyzed in this section uses a construction
that is new, as far as we know, that we call a {\em balanced
factorization}.   This factorization may be of independent interest.

Recall that a {\em polar decomposition} of a matrix $A$ consists
of a unitary matrix $R$ and a positive semidefinite matrix $P$ such
that $A = R P$.  The {\em principal $L$th root} of a complex number
whose expression in polar coordinates is $r e^{\theta i}$ is
$r^{1/L} e^{\theta i/L}$.
The {\em principal $L$th root} of a matrix $A$ is the matrix
$B$ such that $B^L = A$, and each eigenvalue of $B$ is the principal
$L$th root of the corresponding eigenvalue of $A$.  

\begin{definition}
If $A$ be a matrix with polar decomposition $RP$, then
$A$ has the \emph{balanced factorization}
$A = A_1,...,A_L$
where for each $i$,
\[
A_i = R^{1/L} P_i,  \text{\ with\ } P_i = R^{(L-i)/L} P^{1/L} R^{-(L-i)/L},
\]
and each of the $L$th roots is the principal $L$th root.
\end{definition}

The motivation for balanced factorization is as follows.  We want
each factor to do a $1/L$ fraction of the total amount of rotation,
and a $1/L$ fraction of the total amount of scaling.  However, the
scaling done by the $i$th factor should be done in directions that
take account of the partial rotations done by the other factors.  
The following is the key property of the balanced factorization;
its proof is in Appendix~\ref{a:balanced}.

\begin{lemma}
\label{l:balanced}
If $\sigma_1,...,\sigma_d$ are the singular values of
$A$, and $A_1,...,A_L$ is a balanced factorization of $A$, then
the following hold:
(a) $A = \prod_{i=1}^L A_i$;
(b) for each $i \in \{ 1,...,L\}$,
$\sigma_1^{1/L},...,\sigma_d^{1/L}$ are the singular values
 of $A_i$.
\end{lemma}

\subsection{Procedure and upper bound}
\label{s:powerp.procedure}

The following is the {\em power projection algorithm}.  It
has a positivity parameter $\gamma > 0$, and
uses
$
\cH = \{ A : \forall u \mbox{ s.t. } || u || = 1,\; u^{\top} A u \geq \gamma \}
$
as its ``hypothesis space''.  
First, it initializes $\Theta_i^{(0)} = \gamma^{1/L} I$ 
     for all $i \in \{ 1,..., L\}$.  
Then, for each $t$, it does the following.
\begin{itemize}
\item {\bf Gradient Step.} For each $i \in \{ 1,..., L\}$, update:
\[
\Theta_i^{(t+1/2)} = 
\Theta_i^{(t)} 
 - \eta (\Theta_{i+1:L}^{(t)})^{\top} \left(\Theta_{1:L}^{(t)} - \Phi  \right) 
         (\Theta_{1:i-1}^{(t)})^{\top}.
\]
\item {\bf Power Project.} Compute the projection $\Psi^{(t + 1/2)}$ (w.r.t. the Frobenius norm) of
$\Theta_{1:L}^{(t + 1/2)}$ onto $\cH$.
\item {\bf Factor.}  Let $\Theta_1^{(t+1)},...,\Theta_L^{(t+1)}$ 
be the balanced factorization
of $\Psi^{(t + 1/2)}$, so that $\Psi^{(t + 1/2)} = \Theta_{1:L}^{(t+1)}$.
\end{itemize}

\begin{theorem}
\label{t:succeeds.powerp}
For any $\Phi$ such that $u^{\top} \Phi u > \gamma$ for all
unit-length $u$, the power projection algorithm produces
$\Theta^{(t)}$ with $\ell(\Theta^{(t)}) \leq \epsilon$ in
$\mathrm{\poly} (d, || \Phi ||_F, \frac{1}{\gamma}) \log
(1/\epsilon)$ iterations.
\end{theorem}

\subsection{Proof of Theorem~\ref{t:succeeds.powerp}}
\label{s:powerp.analysis}

\begin{lemma}
For all $t$, $\Theta_{1:L}^{(t)} \in \cH$.
\end{lemma}
\begin{proof}
$\Theta_{1:L}^{(0)} = \gamma I \in \cH$, and, for all $t$,
$\Psi^{(t + 1/2)}$ is obtained by projection onto $\cH$, and
$\Theta_{1:L}^{(t+1)} = \Psi^{(t + 1/2)}$.
\end{proof}

\begin{definition}
The exponential of a matrix $A$ is
$\exp(A) \eqdef \sum_{k=0}^{\infty} \frac{1}{k!} A^k$, and 
$B$ is a logarithm of $A$ if $A = \exp(B)$.
\end{definition}

\begin{lemma}[\citealt{culver1966existence}]
\label{l:real.log}
A real matrix has a real logarithm if and only if it is invertible and
each Jordan block belonging to a negative eigenvalue occurs an even
number of times.
\end{lemma}

\begin{lemma}
For all $t$, $\Theta_{1:L}^{(t)}$ has a real $L$th root.
\end{lemma}
\begin{proof}
Since $\Theta_{1:L}^{(t)} \in \cH$ implies
$u^{\top} \Theta_{1:L}^{(t)} u > 0$ for all $u$,
$\Theta_{1:L}^{(t)}$ does not have a negative eigenvalue
and is invertible.   By Lemma~\ref{l:real.log},
$\Theta_{1:L}^{(t)}$ has a real logarithm.  Thus,
its real $L$th root can be constructed via
$\exp(\log(\Theta_{1:L}^{(t)})/L)$.
\end{proof}

The preceding lemma implies that the algorithm is well-defined, since
all of the required roots can be calculated.

\begin{lemma}
$\cH$ is convex.
\end{lemma}
\begin{proof}
Suppose $A$ and $B$ are in $H$ and $\lambda \in (0,1)$.  We have
\[
u^{\top} ( \lambda A + (1 - \lambda) B) u 
 = \lambda u^{\top} A u + (1 - \lambda) u^{\top} B u \geq \gamma.
\]
\end{proof}

\begin{lemma}
For all $A \in \cH$, $\sigma_{\min}(A) \geq \gamma$.
\end{lemma}
\begin{proof}
Let $u$ and $v$ be singular vectors such that $u^{\top} A v = \sigma_{\min}(A)$.  
\[
\gamma \leq v^{\top} A v = \sigma_{\min}(A) v^{\top} u \leq \sigma_{\min}(A).
\]
\end{proof}

\begin{lemma}
\label{l:sigmamin.low}
For all $t$,
$
\sigma_{\min}(\Theta_i^{(t)}) \geq \gamma^{1/L}.
$
\end{lemma}
\begin{proof}
First, $\sigma_{\min}(\Theta_i^{(0)}) =  \gamma^{1/L} \geq \gamma^{1/L}$.

Now consider $t > 0$.  Since $\Psi^{(t - 1/2)}$ was
projected into $\cH$, we have $\sigma_{\min}(\Psi^{(t - 1/2)})
\geq \gamma$.  Lemma~\ref{l:balanced} then completes the proof.
\end{proof}

Define
$U(t) = \max \left\{ \max_{s \leq t} \max_i || \Theta_i^{(s)} ||_2, 
             || \Phi ||_2^{1/L}
            \right\}$,
$B(t) = \min_{s \leq t} \min_i \sigma_{\min} (\Theta_i^{(s)})$,
and recall that $\ell(t) = || \Theta_{1:L}^{(t)} - \Phi ||_F^2.$

Arguing as in the initial portion of
Section~\ref{s:identity.analysis}, as long as
\begin{equation}
\label{e:eta_t.necc.posdef}
\eta \leq \frac{ 1 }{3 L d^5 U(t)^{2L}}
\end{equation}
we have
$\ell(t+1/2)
 \leq \left(1 - \eta L B(t)^{2 L} \right) \ell(t)$  (see Equation~\ref{e:loss.recurrence}).
Lemma~\ref{l:sigmamin.low} gives $B(t) \geq \gamma^{1/L}$, so
$\ell(t+1/2)
 \leq \left(1 - \eta L \gamma^2 \right) \ell(t).$
Since 
$\Psi^{(t + 1/2)}$ is the projection of 
$\Theta_{1:L}^{(t+1/2)}$ onto a convex set $\cH$ that contains
$\Phi$, and
$\Theta_{1:L}^{(t+1)} = \Psi^{(t + 1/2)}$,
(\ref{e:eta_t.necc.posdef}) implies
\begin{equation}
\label{e:loss.one.step}
\ell(t+1) \leq \ell(t+1/2)
 \leq \left(1 - \eta L \gamma^2 \right) \ell(t).
\end{equation}

Next, we prove an upper bound on $U$.
\begin{lemma}
\label{l:Ubound}
For all $t$, 
$U(t) \leq \left(\sqrt{\ell(t)} + || \Phi ||_F\right)^{1/L}$.
\end{lemma}
\begin{proof}
Recall that $\ell(t) = || \Theta_{1:L}^{(t)} - \Phi ||_F^2$.
By the triangle inequality, 
$|| \Theta_{1:L}^{(t)} ||_F 
  \leq \sqrt{\ell(t)} + || \Phi ||_F.$
Thus
$
|| \Theta_{1:L}^{(t)} ||_2 
  \leq \sqrt{\ell(t)} + || \Phi ||_F.$
By Lemma~\ref{l:balanced}, for all $i$, we have
$
|| \Theta_i^{(t)} ||_2 
  \leq \left(\sqrt{\ell(t)} + || \Phi ||_F\right)^{1/L}.$
Since $|| \Phi ||_2 \leq || \Phi ||_F$, this completes
the proof.
\end{proof}

Note that the triangle inequality implies that
$\ell(0) 
\leq || \Theta_{1:L}^{(0)} ||_F^2 + || \Phi ||_F^2
\leq \gamma^2 d + || \Phi ||_F^2$.
Since $\sigma_{\min}(\Phi) \geq \gamma$, we have
$|| \Phi ||_F^2 \geq \gamma^2 d$, 
so $\ell(t) \leq 2 || \Phi ||_F^2$ and 
$U(t) \leq (3 || \Phi ||_2)^{1/L}$.
Now, if we set
$\eta = \frac{ 1 }{c L d^5 || \Phi ||_F^2},$
for a large enough absolute constant $c$,  then
(\ref{e:eta_t.necc.posdef}) is satisfied, so that
(\ref{e:loss.one.step}) gives
$\ell(t+1) 
 \leq \left(1 - \frac{ \gamma^2}{c d^5 || \Phi ||_F^2} \right) \ell(t)$
and the power projection algorithm achieves
$\ell(t+1) \leq \epsilon$ after
\begin{align*}
& O\left(\frac{d^5 || \Phi ||_F^2 }{\gamma^2} 
   \log \left( \frac{\ell(0)}{\epsilon} \right)\right) 
    \\ &  
 \!=\!
  O\left(\frac{d^5 || \Phi ||_F^2 }{\gamma^2} 
  \log \left( \frac{|| \Phi ||_F^2}{\epsilon} \right)\right)
\end{align*}
updates.

\section{Failure}
\label{s:failure}

In this section, we show that positive definite $\Phi$ are 
necessary for several gradient descent algorithms with different kinds of regularization to minimize the loss.
One family of algorithms that we will analyze
is parameterized by a function $\psi$
mapping the number of inputs $d$ and the number of
layers $L$ to a radius $\psi(d,L)$, step sizes $\eta_t$ and initialization parameter $\gamma \geq 0$.
In particular, a {\em $\psi$-step-and-project} algorithm is any instantiation of  the  following algorithmic template.

Initialize each $\Theta_i^{(0)} =  \gamma^{1/L} I $ for some $\gamma \geq 0$ and  iterate:
\begin{itemize}
\item {\bf Gradient Step.} For each $i \in \{ 1,..., L\}$, update:
\[
\Theta_i^{(t+1/2)} = 
\Theta_i^{(t)} 
 - \eta_t (\Theta_{i+1:L}^{(t)})^{\top} \left(\Theta_{1:L}^{(t)} - \Phi  \right) 
         (\Theta_{1:i-1}^{(t)})^{\top}.
\]
\item {\bf Project.} Set each $\Theta_i^{t+1}$ to the projection of $\Theta_i^{t+1/2}$ onto
     $\{ A : || A - I ||_2 \leq \psi(d,L) \}$.
\end{itemize}

We will also show that \emph{Penalty Regularized Gradient Descent} which uses gradient descent with any step sizes $\eta_t$ 
on the regularized objective $\ell(\Theta) + \frac{\kappa}{2} \sum_i || I - \Theta ||_F^2$ also fails to minimize the loss.

Both results use the simple observation that when $\Theta_{1:L}$ and $\Phi$ are mutually diagonalizable then
\[
|| \Theta_{1:L} - \Phi ||_F^2
 = || U^{\top} \hat D U - U^{\top} D U ||_F^2 
=  \sum_{j=1}^d ({\hat D}_{jj} - D_{jj})^2,
\]
where the $D_{ii}$ are the eigenvalues of $\Phi$.

%
%

\begin{theorem}
\label{t:pen.grad.descent}
If the least squares matrix $\Phi$ is symmetric then
Penalty Regularized Gradient Descent produces hypotheses $\Theta^{(t)}_{1:L}$ that are commuting normal with 
$\Phi$. 

In addition, if $\Phi$ has a negative eigenvalue $-\lambda$ and $L$ is even, then $\ell(\Theta^{(t)}) \geq \lambda^2/2$ for all $t$.
\end{theorem}
\begin{proof}
For all $t$, Penalty Regularized Gradient Descent produces
$
\Theta_i^{(t+1)} = (1 - \kappa) \Theta_i^{(t)} 
 + \kappa I
 - \eta_t (\Theta_{i+1:L}^{(t)})^{\top} \left(\Theta_{1:L}^{(t)} - \Phi \right) 
         (\Theta_{1:i-1}^{(t)})^{\top}.
$
Thus, by induction, the $\Theta_i^{(t)}$ are 
matrix polynomials of $\Phi$, and therefore
they are all commuting normal.
As in Lemmas~\ref{l:commute} and~\ref{l:independent.eigenvalues}
each  $\Theta_i^{(t)}$ is the same $U^{\top} {\tilde D}^{(t)} U$
and $\Theta_{1:L}^{(t)} = U^{\top} ({\tilde D}^{(t)})^L U$. 
Since $L$ is even, each 
$({\tilde D}^{(t)})^L_{jj} \geq 0$, 
so 
$\ell(\Theta^{(t)}) = \frac{1}{2} || \Theta_{1:L}^{(t)} - \Phi ||_F^2\geq \lambda^2/2.$
\end{proof}

To analyze step-and-project algorithms, it is helpful
to first characterize the project step
(see also \citep{lefkimmiatis2013hessian}).
\begin{lemma}
\label{l:project}
Let $X$ be a symmetric matrix
and let $U^{\top} D U$ be its diagonalization.

For $a > 0$, let $Y$ be the Frobenius norm projection of $X$ onto 
$
      \cB_a = \{ A : A \mbox{ is symmetric psd and } 
                || A - I ||_2 \leq a \}.
$
Then $Y = U^{\top} \tD U$ where $\tilde{D}$ is obtained from $D$ 
by projecting all of its diagonal elements onto $[1-a,1+a]$.

Thus $\{X, Y \}$ are symmetric commuting normal matrices.
\end{lemma}
\begin{proof}
First, if $X \in \cB_{a}$, then $Y = X$ and we are done.
 
Assume $X \not\in \cB_{a}$.    
Clearly $U^{\top} \tilde{D} U \in \cB_{a}$, so we
just need to show that any member of $\cB_{a}$ is at least as
far from $X$ as $U^{\top} \tilde{D} U$ is.  Let $\Lambda$ be the multiset
of eigenvalues of $X$ (with repetitions) that are not in $[1-a,1+a]$,
and for each $\lambda \in \Lambda$, let
$e_{\lambda}$ be the adjustment to $\lambda$ necessary
to bring it to $[1-a,1+a]$; i.e., so that $\lambda + e_{\lambda}$ is
the projection of $\lambda$ onto $[1-a,1+a]$.

If $u_{\lambda}$ is the eigenvector associated with $\lambda$,
we have $U^{\top} \tilde{D} U - X = \sum_{\lambda \in \Lambda} e_{\lambda} u_{\lambda} u_{\lambda}^{\top}$,
so that 
$
|| U^{\top} \tilde{D} U - X ||_F^2 = \sum_{\lambda \in \Lambda} e_{\lambda}^2.
$

Let $Z$ be an arbitrary member of $\cB_{a}$.  We would like
to show that $|| Z - X ||_F^2 \geq \sum_{\lambda \in \Lambda} e_{\lambda}^2.$
Since $Z \in \cB_{a}$, we have $|| Z - I ||_2 \leq a$.  
$|| Z - I ||_2$ is the largest singular value of 
$Z - I$ so, for
any unit length vector, in particular
some $u_{\lambda}$ for $\lambda \in \Lambda$,
$| u_{\lambda}^{\top} (Z - I) u_{\lambda} |  =
| u_{\lambda}^{\top} Z u_{\lambda} - 1 | \leq a$,
which implies $u_{\lambda}^{\top} Z u_{\lambda} \in [1-a,1+a]$.
Since $U$ is unitary $U^{\top} (X - Z) U$ has the same
eigenvalues as $X - Z$, 
and, since the Frobenius norm is a function of the eigenvalues,
$|| U^{\top} (X - Z) U ||_F = || X - Z ||_F$.
But since $u_{\lambda}^{\top} Z u_{\lambda} \in [1-a,1+a]$ 
for all $\lambda \in \Lambda$, just summing over the
diagonal elements, we get 
$|| U^{\top} (X - Z) U ||_F^2 \geq \sum_{\lambda \in \Lambda} e_{\lambda}^2$,
completing the proof.
\end{proof}


\begin{theorem}
\label{t:step.project}
If the least squares matrix $\Phi$ is symmetric then 
$\psi$-step-and-project algorithms produce hypotheses $\Theta^{(t)}_{1:L}$ that are commuting normal with 
$\Phi$. 

In addition, if $\Phi$ has a negative eigenvalue $-\lambda$ and either $L$ is even or 
$\psi(L,d) \leq 1$, then $\ell(\Theta^{(t)}) \geq \lambda^2/2$ for all $t$.
\end{theorem}

%

\begin{proof}  
As in Lemmas~\ref{l:commute} and~\ref{l:independent.eigenvalues},
the $\Theta_i^{(t+1/2)}$  are identical and mutually diagonalizable with $\Phi$.
Lemma~\ref{l:project} shows that this is preserved by the projection step.
Thus 
there is a real diagonal ${\tilde D}^{(t)}$ such that
each $\Theta_i^{(t)} = U^{\top} D_i^{(t)} U$, so 
$\Theta^{(t)}_{1:L} = U^{\top} ({\tilde D}^{(t)})^L U$.

When $L$ is even, each 
$({\tilde D}^{(t)})^L)_{j,j} \geq 0$.
When $\psi(d,L) \leq 1$ then the projection ensures that the elements of ${\tilde D}^{(t)}$ are non-negative,
and thus each $({\tilde D}^{(t)})^L)_{j,j} \geq 0$.
In either case, 
$\ell(\Theta^{(t)}) = \frac{1}{2} || \Theta_{1:L}^{(t)} - \Phi ||_F^2\geq \lambda^2/2.$
\end{proof}

One choice of $\Phi$ that satisfies the
requirements of Theorems~\ref{t:pen.grad.descent} and~\ref{t:step.project} 
is
$\Phi = \diag(-\lambda,1,1,...,1)$.  For
constant $\lambda$, the loss of $\Theta^{(0)} = (I,I,...,I)$
is a constant for this target.  Another
choice is $\Phi = \diag(-\lambda,-\lambda,1,1,...,1)$,
which has a positive determinant.

Our proof of failure to minimize the loss exploits the fact that the layers are initialized to multiples of the identity. 
Since the training process is a continuous function of the initial solution, this implies that 
any convergence to a good solution will be very slow if the initializations are sufficiently close to the identity.

%
%
%

\section*{Acknowledgements}

We thank Yair Carmon, Nigel Duffy, Matt Feiszli,
Roy Frostig, Vineet Gupta, Moritz Hardt, Tomer
Koren, Antoine Saliou,
Hanie Sedghi, Yoram Singer and Kunal Talwar for valuable
conversations.

Peter Bartlett gratefully acknowledges the support of the NSF through
grant IIS-1619362 and of the Australian Research Council through an
Australian Laureate Fellowship (FL110100281) and through the
Australian Research Council Centre of Excellence for Mathematical and
Statistical Frontiers (ACEMS).



\bibliography{refs}
\bibliographystyle{icml2018}

\appendix

\section{Proof of Lemma~\protect\ref{l:gradient.hessian}}
\label{a:gradient.hessian}

We rely on the following facts \citep{Horn:1986:TMA:19572,harville1997matrix}.
\begin{lemma}
\label{l:tools}
  For compatible matrices (and, where $m,n,p,q,r,s$ are mentioned,
  $A\in\Re^{m\times n}$, $B\in\Re^{p\times q}$, $X \in\Re^{r\times s}$):
    \begin{align*}
      A\otimes (B\otimes E) & = (A\otimes B)\otimes E, \\
      AC\otimes BD & = (A\otimes B)(C\otimes D), \\
      (A\otimes B)^\top &= A^\top \otimes B^\top, \\
      \vecrm(AXB) &= (B^\top\otimes A)\vecrm(X), \\
      T_{m,n}\vecrm(A) &\defeq \vecrm(A^\top), \\
      T_{n,m}T_{m,n}&= I_{mn}, \\
      T_{m,n} & = T_{n,m}^\top, \\
      T_{1,n} & = T_{n,1} = I_n, \\
      D_X(A(B(X))) &= D_B(A(B(X))) D_X(B(X)), \\
      D_X(A(X)B(X)) &= (B(X)^\top \otimes I_m) D_X A(X) 
                          + (I_q \otimes A(X)) D_X B(X), \\
      D_X(A(X)^T) & = T_{n,m} D_X(A(X)), \\
      D_X (AXB) &= B^\top\otimes A, \\
      D_A (A\otimes B)
        &= (I_n\otimes T_{q,m}\otimes I_p)(I_{mn}\otimes\vecrm(B)) \\
        &= (I_{nq}\otimes T_{m,p})(I_n\otimes\vecrm(B)\otimes I_m), \\
      D_B (A\otimes B)
        &= (I_n\otimes T_{q,m}\otimes I_p)(\vecrm(A)\otimes I_{pq}) \\
        &= (T_{p,q}\otimes I_{mn})(I_q\otimes\vecrm(A)\otimes I_p).
    \end{align*}
\end{lemma}

\medskip
Armed with Lemma~\ref{l:tools}, we now prove Lemma~\ref{l:gradient.hessian}.  We
have
  \begin{align*}
     D_{\Theta_i} f_{\Theta}(x)
       &= D_{\Theta_i} \left(\Theta_{i+1:L}\Theta_i\Theta_{1:i-1}x\right)
       = \left(\Theta_{1:i-1}x\right)^\top\otimes\Theta_{i+1:L}.
   \end{align*}
Again, from Lemma~\ref{l:tools}
\begin{align}
  \nonumber
     D_{\Theta_i} \left(D_{\Theta_j} f_{\Theta}(x)\right)  
     \nonumber
       &= D_{\Theta_i} \left(\left(
         \Theta_{1:j-1}x\right)^\top\otimes\Theta_{j+1:L}\right) \\
     \nonumber
       &= D_{\Theta_{1:j-1}x} \left(\left(
         \Theta_{1:j-1}x\right)^\top\otimes\Theta_{j+1:L}\right)
         D_{\Theta_i} \left(\Theta_{1:j-1}x\right) \\
     \nonumber
   &       \;\;\;\;\;\;\;\;\; \mbox{(by the chain rule, since $i < j$)} \\
     \label{e:diff.transpose}
       &= D_{\Theta_{1:j-1}x} \left(\left(
         \left(
         \Theta_{1:j-1}x\right) \otimes \Theta_{j+1:L}^\top 
         \right)^\top
         \right)
         \left(\left(\Theta_{1:i-1}x\right)^\top\otimes \Theta_{i+1:j-1}\right).
\end{align}
Define $P = \Theta_{1:j-1}x$ and
$Q = \Theta_{j+1:L}$, so that $P \in \Re^{d \times 1}$ 
 and $Q \in \Re^{d \times d}$.  
We have
\begin{align*}
D_{P} \left(\left( P \otimes Q^\top \right)^\top \right)
&  = T_{d^2, d} D_{P} \left(P \otimes Q^\top \right) \\
&  = T_{d^2, d} (I_1 \otimes T_{d,d} \otimes I_{d})
      (I_{d} \otimes \vecrm(Q^T)) \\
&  = T_{d^2, d} (T_{d,d} \otimes I_{d}) (I_{d} \otimes \vecrm(Q^\top)). \\
\end{align*}
Substituting back into (\ref{e:diff.transpose}), we get
\begin{align*}
  D_{\Theta_i} \left(D_{\Theta_j} f_{\Theta}(x)\right) 
       &=   T_{d^2, d} (T_{d,d} \otimes I_{d}) (I_{d} \otimes \vecrm(\Theta_{j+1:L}^\top)) 
         \left(\left(\Theta_{1:i-1}x\right)^\top\otimes \Theta_{i+1:j-1}\right). 
\end{align*}
The product rule in Lemma~\ref{l:tools} gives, for each $i$,
\begin{align*}
  D_{\Theta_i} \ell\left(f_{\Theta}\right) 
 & = \Expect (D_{\Theta_i} (\ell(f_{\Theta}(X))) \\
 & = \Expect (D_{\Theta_i} (\frac{1}{2} (f_{\Theta}(X)-\Phi X)^{\top} (f_{\Theta}(X)-\Phi X))) \\
 &= \Expect (((\Theta_{1:L} - \Phi) X )^{\top} D_{\Theta_i} f_{\Theta}(X)) \\
 &= \Expect \left(((\Theta_{1:L} - \Phi) X )^{\top}
         \left( \left(\Theta_{1:i-1}X\right)^\top\otimes\Theta_{i+1:L}\right) \right) \\
       &= \Expect \left((I_1 \otimes ((\Theta_{1:L} - \Phi) X )^{\top})
         \left( \left(\Theta_{1:i-1}X\right)^\top\otimes\Theta_{i+1:L}\right) \right) \\
       &= \Expect \left(
         \left( \left(\Theta_{1:i-1}X\right)^\top\otimes ((\Theta_{1:L} - \Phi) X )^{\top} \Theta_{i+1:L}\right) \right) \\
       &= \Expect
         \left( \left(X^{\top}\Theta_{1:i-1}^{\top}\right) \otimes
         \left(X^{\top} (\Theta_{1:L} - \Phi)^{\top} \Theta_{i+1:L}\right)
         \right) \\
       &= \Expect \left(( X^{\top} \otimes X^{\top} )
         \left(\Theta_{1:i-1}^{\top} \otimes (\Theta_{1:L} - \Phi)^{\top} \Theta_{i+1:L}\right)  \right) \\
       &= \Expect \left(( X \otimes X)\vecrm(1)\right)^{\top} 
         \left(\Theta_{1:i-1}^{\top} \otimes (\Theta_{1:L} - \Phi)^{\top} \Theta_{i+1:L}\right) \\
       &= \Expect \left(\vecrm(XX^\top)\right)^{\top} 
         \left(\Theta_{1:i-1}^{\top} \otimes (\Theta_{1:L} - \Phi)^{\top} \Theta_{i+1:L}\right) \\
       &= (\vecrm( I_d ))^T
         \left(\Theta_{1:i-1}^{\top} \otimes (\Theta_{1:L} - \Phi)^{\top} \Theta_{i+1:L}\right) .
   \end{align*}
Hence,
\begin{align*}
 \left(D_{\Theta_i} \ell\left(f_{\Theta}\right)\right)^\top
   & = \left(\Theta_{1:i-1} \otimes \Theta_{i+1:L}^\top
     (\Theta_{1:L} - \Phi)\right) (\vecrm( I_d )) \\
   & = \vecrm\left( \Theta_{i+1:L}^\top
     (\Theta_{1:L} - \Phi) I_d \Theta_{1:i-1}^\top\right).
\end{align*}

Also, recalling that $i < j$, we have
    \begin{align}
    \nonumber
 D_{\Theta_j} D_{\Theta_i} \ell\left(f_{\Theta}\right) 
    \nonumber
   & = D_{\Theta_j} \left( (\vecrm( I_d ))^T
          \left(\Theta_{1:i-1}^{\top} \otimes (\Theta_{1:L} - \Phi)^{\top} \Theta_{i+1:L}\right)  \right) \\
   \nonumber
   & = (I_{d^2} \otimes (\vecrm( I_d ))^T)
  D_{\Theta_j} 
         \left(\Theta_{1:i-1}^{\top} \otimes (\Theta_{1:L} - \Phi)^{\top} \Theta_{i+1:L}\right)   \\
   &  = (I_{d^2} \otimes (\vecrm( I_d ))^T)
      \left( I_d \otimes T_{d,d} \otimes I_d \right)
      \left( \vecrm(\Theta_{1:i-1}^{\top}) \otimes I_{d^2} \right) 
   D_{\Theta_j}  \left( (\Theta_{1:L} - \Phi)^{\top} \Theta_{i+1:L}\right).
  \nonumber
    \end{align}
 Continuing with the subproblem, 
\begin{align*}
D_{\Theta_j}  \left( (\Theta_{1:L} - \Phi)^{\top} \Theta_{i+1:L}\right)
   &  = 
     (\Theta_{i+1:L}^{\top} \otimes I_d)  D_{\Theta_j}  \left( (\Theta_{1:L} - \Phi)^{\top} \right) \\
& \hspace{0.3in}
     + (I_d \otimes (\Theta_{1:L} - \Phi)^{\top}) D_{\Theta_j}  \left( \Theta_{i+1:L}\right) \\
   &  = 
     (\Theta_{i+1:L}^{\top} \otimes I_d)  D_{\Theta_j}  \left( \Theta_{1:L}^{\top} \right) \\
& \hspace{0.3in}
     + (I_d \otimes (\Theta_{1:L} - \Phi)^{\top}) D_{\Theta_j}  \left( \Theta_{i+1:L}\right) \\
    &  = 
     (\Theta_{i+1:L}^{\top} \otimes I_d)  
     \left( \Theta_{j+1:L} \otimes \Theta_{1:j-1}^{\top} \right) D_{\Theta_j}(\Theta_j^{\top}) \\
& \hspace{0.3in}
     + (I_d \otimes (\Theta_{1:L} - \Phi)^{\top}) 
     \left( \Theta_{i+1:j-1}^{\top} \otimes \Theta_{j+1:L} \right) \\
    &  = 
     (\Theta_{i+1:L}^{\top} \otimes I_d)  
     \left( \Theta_{j+1:L} \otimes \Theta_{1:j-1}^{\top} \right) T_{d,d} \\
& \hspace{0.3in}
     + (I_d \otimes (\Theta_{1:L} - \Phi)^{\top}) 
     \left( \Theta_{i+1:j-1}^{\top} \otimes \Theta_{j+1:L} \right) \\
    &  = 
     \left( \Theta_{i+1:L}^{\top} \Theta_{j+1:L} \otimes \Theta_{1:j-1}^{\top} \right) T_{d,d} \\
& \hspace{0.3in}
     +
     \left( \Theta_{i+1:j-1}^{\top} \otimes (\Theta_{1:L} - \Phi)^{\top} \Theta_{j+1:L} \right).
     \end{align*}
Finally,
\begin{align*}
D_{\Theta_i} D_{\Theta_i} \ell\left(f_{\Theta}\right)
& = D_{\Theta_i} \left( (\vecrm( I_d ))^T
         \left(\Theta_{1:i-1}^{\top} \otimes (\Theta_{1:L} - \Phi)^{\top} \Theta_{i+1:L}\right) \right) \\
& = (I_{d^2} \otimes (\vecrm( I_d ))^T) 
         D_{\Theta_i} \left(\Theta_{1:i-1}^{\top} \otimes (\Theta_{1:L} - \Phi)^{\top} \Theta_{i+1:L}\right) \\
   &  = (I_{d^2} \otimes (\vecrm( I_d ))^T)
      \left( I_d \otimes T_{d,d} \otimes I_d \right) \\
& \hspace{0.3in}
      \left( \vecrm(\Theta_{1:i-1}^{\top}) \otimes I_{d^2} \right)
    D_{\Theta_i}  \left( (\Theta_{1:L} - \Phi)^{\top} \Theta_{i+1:L}\right)
\end{align*}
and 
\begin{align*}
 D_{\Theta_i}  \left( (\Theta_{1:L} - \Phi)^{\top} \Theta_{i+1:L}\right) 
   &  = 
     (\Theta_{i+1:L}^{\top} \otimes I_d)  D_{\Theta_i}  \left( (\Theta_{1:L} - \Phi)^{\top} \right)  \\
   &  = 
     (\Theta_{i+1:L}^{\top} \otimes I_d)  D_{\Theta_i}  \left( \Theta_{1:L}^{\top} \right)  \\
    &  = 
     (\Theta_{i+1:L}^{\top} \otimes I_d)  
     \left( \Theta_{i+1:L} \otimes \Theta_{1:i-1}^{\top} \right) D_{\Theta_i}(\Theta_i^{\top}) \\
    &  = 
     (\Theta_{i+1:L}^{\top} \otimes I_d)  
     \left( \Theta_{i+1:L} \otimes \Theta_{1:i-1}^{\top} \right) T_{d,d} \\
    &  = 
     \left( \Theta_{i+1:L}^{\top} \Theta_{i+1:L} \otimes \Theta_{1:i-1}^{\top} \right) T_{d,d}.
     \end{align*}

\section{Proof of Lemma~\protect\ref{l:smooth}}
\label{a:smooth}

We have
\begin{equation}
\label{e:byPerLayer}
|| \nabla^2 ||_F^2
 = 2 \sum_{i < j} || D_{\Theta_j} D_{\Theta_i} \ell(f_{\Theta}) ||_F^2
    +  \sum_i || D_{\Theta_i} D_{\Theta_i} \ell(f_{\Theta}) ||_F^2.
\end{equation}
Let's start with the easier term.  Choose $\Theta$
such that $|| \Theta_i - I ||_2 \leq z$ for all $i$.  We have
\begin{align*}
|| D_{\Theta_i} D_{\Theta_i} \ell\left(f_{\Theta}\right) ||_F 
& = \big|\big|
(I_{d^2} \!\otimes\! (\vecrm( I_d ))^{\top})
      \left( I_d \!\otimes\! T_{d,d} \!\otimes\! I_d \right)
      \left( \vecrm(\Theta_{1:i-1}^{\top}) \!\otimes\! I_{d^2} \right) \\
& \hspace{0.35in}
   \left( \Theta_{i+1:L}^{\top} \Theta_{i+1:L} \otimes \Theta_{1:i-1}^{\top} \right) T_{d,d} \big|\big|_F  \\
& \leq \left|\left|
(I_{d^2} \otimes (\vecrm( I_d ))^{\top})
      \left( I_d \otimes T_{d,d} \otimes I_d \right)
   \right|\right|_F \\
& \hspace{0.35in}
   \times 
   \left|\left|
      \left( \vecrm(\Theta_{1:i-1}^{\top}) \!\otimes\! I_{d^2} \right)
     \left( \Theta_{i+1:L}^{\top} \Theta_{i+1:L} \!\otimes\! \Theta_{1:i-1}^{\top} \right) T_{d,d} \right|\right|_F \\
& = d^{3/2}
   \left|\left|
      \left( \vecrm(\Theta_{1:i-1}^{\top}) \otimes I_{d^2} \right) \right.\right. \\
 & \left.\left.
   \hspace{0.6in}
     \left( \Theta_{i+1:L}^{\top} \Theta_{i+1:L} \otimes \Theta_{1:i-1}^{\top} \right) T_{d,d} \right|\right|_F \\
& \leq d^{3/2} 
   \left|\left|
     \left( \vecrm(\Theta_{1:i-1}^{\top}) \otimes I_{d^2} \right) \right|\right|_F \\
& \hspace{0.35in}
    \times
       \left|\left|
     \left( \Theta_{i+1:L}^{\top} \Theta_{i+1:L} \otimes \Theta_{1:i-1}^{\top} \right) T_{d,d} \right|\right|_F \\
& = d^{7/2}
   \left|\left|
     \vecrm(\Theta_{1:i-1}^{\top})  \right|\right|_F
       \left|\left|
     \left( \Theta_{i+1:L}^{\top} \Theta_{i+1:L} \!\otimes\! \Theta_{1:i-1}^{\top} \right) T_{d,d} \right|\right|_F \\
& = d^{7/2}
   \left|\left|
     \Theta_{1:i-1}  \right|\right|_F
       \left|\left|
     \left( \Theta_{i+1:L}^{\top} \Theta_{i+1:L} \otimes \Theta_{1:i-1}^{\top} \right) T_{d,d} \right|\right|_F \\
& \leq d^4
   \left|\left|
     \Theta_{1:i-1}  \right|\right|_2
       \left|\left|
     \left( \Theta_{i+1:L}^{\top} \Theta_{i+1:L} \otimes \Theta_{1:i-1}^{\top} \right) T_{d,d} \right|\right|_F \\
& \leq d^4 (1 + z)^{i-1}
       \left|\left|
     \left( \Theta_{i+1:L}^{\top} \Theta_{i+1:L} \otimes \Theta_{1:i-1}^{\top} \right) T_{d,d} \right|\right|_F \\
& = d^4 (1 + z)^{i-1}
       \left|\left|
     \left( \Theta_{i+1:L}^{\top} \Theta_{i+1:L} \otimes \Theta_{1:i-1}^{\top} \right) \right|\right|_F \\
& = d^4 (1 + z)^{i-1}
       \left|\left|
      \Theta_{i+1:L}^{\top} \Theta_{i+1:L}  \right|\right|_F 
    \times
       \left|\left|
    \Theta_{1:i-1}^{\top} \right|\right|_F \\
& \leq d^5 (1 + z)^{i-1}
       \left|\left|
      \Theta_{i+1:L}^{\top} \Theta_{i+1:L}  \right|\right|_2
    \times
       \left|\left|
    \Theta_{1:i-1}^{\top} \right|\right|_2 \\
& \leq d^5 (1 + z)^{2 (L-1)}.
\end{align*}
Similarly,
\begin{align*}
|| D_{\Theta_j} D_{\Theta_i} \ell\left(f_{\Theta}\right) ||_F
  & =     \big|\big|
      (I_{d^2} \!\otimes\! (\vecrm( I ))^{\top})
      \left( I_d \!\otimes\! T_{d,d} \!\otimes\! I_d \right)
      \left( \vecrm(\Theta_{1:i-1}^{\top}) \!\otimes\! I_{d^2} \right)  \\
 & \hspace{0.3in}
       \bigg(
     \left( \Theta_{i+1:L}^{\top} \Theta_{j+1:L} \otimes \Theta_{1:j-1}^{\top} \right) T_{d,d}  \\
 &    \hspace{0.4in}
     +
     \left( \Theta_{i+1:j-1}^{\top} \otimes (\Theta_{1:L} - \Phi)^{\top} \Theta_{j+1:L} \right)
       \bigg) \big|\big|_F \\
  & \le  d^4 (1 + z)^{i-1}   
   \big|\big|       
     \left( \Theta_{i+1:L}^{\top} \Theta_{j+1:L} \otimes \Theta_{1:j-1}^{\top} \right) T_{d,d} \\
& \hspace{0.4in}
     +
     \left( \Theta_{i+1:j-1}^{\top} \otimes (\Theta_{1:L} - \Phi)^{\top} \Theta_{j+1:L} \right)
        \big|\big|_F \\
  & \leq  d^4 (1 + z )^{i-1}   
   \left(
   \left|\left|       
     \left( \Theta_{i+1:L}^{\top} \Theta_{j+1:L} \otimes \Theta_{1:j-1}^{\top} \right) T_{d,d}
    \right|\right|_F \right. \\
   & \left. \hspace{0.35in}
     +
   \left|\left|       
     \left( \Theta_{i+1:j-1}^{\top} \otimes (\Theta_{1:L} - \Phi)^{\top} \Theta_{j+1:L} \right)
        \right|\right|_F \right) \\
  & \leq  d^4 (1 + z)^{i-1}   
   \big(
    d (1 + z)^{2 L - 1 - i} \\
& \hspace{0.35in}
     +
   \left|\left|       
     \left( \Theta_{i+1:j-1}^{\top} \otimes (\Theta_{1:L} - \Phi)^{\top} \Theta_{j+1:L} \right)
        \right|\right|_F \big) \\
  & =  d^4 (1 + z)^{i-1}   
   \left(
    d (1 + z)^{2 L - 1 - i} \right. \\
& \left. 
   \hspace{0.35in}
     +
   \left|\left|       
     \Theta_{i+1:j-1}
    \right|\right|_F
    \times
   \left|\left| (\Theta_{1:L} - \Phi)^{\top} \Theta_{j+1:L} 
        \right|\right|_F \right) \\
  & \leq  d^4 (1 + z)^{i-1}   
   \left(
    d (1 + z)^{2 L - 1 - i}
     +
    2 d  (1 + z)^{2 L - 1 - i} \right) \\
  & =  3 d^5 (1 + z)^{2 L - 2}.   \\
\end{align*}
Putting these together with (\ref{e:byPerLayer}), we get
$
|| \nabla^2 ||_F^2 \leq L^2 9 d^{10} (1 + z)^{4 L},
$
so that
\[
|| \nabla^2 ||_F \leq 3 L d^5 (1 + z)^{2 L}.
\]

\section{Proof of Lemma~\protect\ref{l:balanced}}
\label{a:balanced}

Recall that a {\em polar decomposition} of a matrix $A$ consists
of a unitary matrix $R$ and a positive semidefinite matrix $P$ such
that $A = R P$.

\begin{lemma}[\citep{horn2013matrix}]
$A$ is a unitary matrix if and only if all of the (complex) eigenvalues $z$
of $A$ have magnitude $1$.
\end{lemma}
%

\begin{lemma}[\citep{horn2013matrix}]
If $A$ is unitary then $A$ is normal.
\end{lemma}

\begin{lemma}[\citep{horn2013matrix}]
\label{l:normal.singular}
If $A$ is normal with eigenvalues $\lambda_1,...,\lambda_d$, 
the singular values of $A$ are $|\lambda_1|,...,|\lambda_d|$.
\end{lemma}

\begin{lemma}
If $A$ is unitary, then $A^{1/L}$ is unitary, and thus $A^{i/L}$ is
unitary for any non-negative integer $i$.
\end{lemma}

\begin{lemma}
If $A$ is invertible and normal with singular values 
$\sigma_1,...,\sigma_d$, then, for any positive integer $L$, the
singular values of $A^{1/L}$ are
$\sigma_1^{1/L},...,\sigma_d^{1/L}$.
\end{lemma}
\begin{proof}
Follows from Lemma~\ref{l:normal.singular} together with the
fact that raising a non-singular matrix to a power results in
raising its eigenvalues to the same power.
\end{proof}

\begin{lemma}[\citep{horn2013matrix}]
If $A = R P$ is the polar decomposition of $A$, then the
singular values of $A$ are the same as the singular values of $P$.
\end{lemma}

\begin{lemma}
\label{l:balanced.sigma}
If $\sigma_1,...,\sigma_d$ are the principal components of
$A$, and $A = \prod_{i=1}^L A_i$ is a balanced factorization
of $A$, then
then $\sigma_1^{1/L},...,\sigma_d^{1/L}$ are the principal
components of $A_i$, for each $i \in \{ 1,...,L\}$.
\end{lemma}
\begin{proof}
The singular values of $A_i = R_i P_i$ are the same
as the singular values of $P_i$, which is similar to
$P^{1/L}$, whose singular values are the $L$th roots 
of the singular values of $P$, which are the same as the
singular values of $A$.
\end{proof}

\begin{lemma}
If $A_1,...,A_L$ is a balanced factorization of $A$, then
\[
A = \prod_{i=1}^L A_i.
\]
\end{lemma}
\begin{proof}
We have
\begin{align*}
A 
 & = R P \\
 & = R^{1/L} R^{1 - 1/L} P^{1/L} P^{1 - 1/L} \\
 & = R^{1/L} R^{1 - 1/L} P^{1/L} R^{-(1 - 1/L)} R^{1 - 1/L} P^{1 - 1/L} \\
 & = R_1 P_1 R^{1 - 1/L} P^{1 - 1/L} \\
 & = A_1 R^{1 - 1/L} P^{1 - 1/L} \\
 & = A_1 R^{1/L} R^{1 - 2/L} P^{1/L} P^{1 - 2/L}
\end{align*}
and so on.
\end{proof}

\end{document}